\documentclass[a4paper, 11pt]{article}

\usepackage{natbib}
\usepackage[T1]{fontenc}
\usepackage{graphicx}
\usepackage{booktabs}
\usepackage[misc]{ifsym}
\usepackage{amsmath}
\usepackage{amsthm}
\usepackage{amssymb}

\usepackage{amsmath,amsfonts, bm}

\def\eqref#1{equation~\ref{#1}}

\def\1{\bm{1}}

\def\vu{{\bm{u}}}
\def\vv{{\bm{v}}}

\def\vx{{\bm{x}}}
\def\vy{{\bm{y}}}
\def\vz{{\bm{z}}}

\def\mW{{\bm{W}}}
\def\mX{{\bm{X}}}
\def\mY{{\bm{Y}}}
\def\mZ{{\bm{Z}}}

\DeclareMathAlphabet{\mathsfit}{\encodingdefault}{\sfdefault}{m}{sl}
\SetMathAlphabet{\mathsfit}{bold}{\encodingdefault}{\sfdefault}{bx}{n}

\def\sR{{\mathbb{R}}}

\usepackage{algorithm}
\usepackage{algorithmic}
\newtheorem{theorem}{Theorem}[section]
\newtheorem{proposition}[theorem]{Proposition}

\newtheorem{remark}[theorem]{Remark}
\usepackage{tikz}
\usetikzlibrary{positioning,arrows.meta}

\begin{document}

\title{Two-Stage Clipping: Linear-Time Approximate and Exact Mixed-Norm Projections for Transformers}

\author{ Michel Barlaud$^1 $, Jean-Paul Marmorat$^2$ \\
$^1 $ I3S Laboratory, CNRS, Universite Côte d'Azur, France \\
$^2 $ Ecole des Mines, PSL, France.\\
}

\maketitle

\begin{abstract}
Large Language Models (LLMs) based on Transformer architectures have
achieved remarkable performance across artificial intelligence tasks,
but their computational cost remains a major bottleneck.
Structured sparsification provides a promising approach to reducing this
cost. In particular, $\ell_{1,\infty}$ projections can induce
column-wise sparsity in model matrices, but existing proximal methods
(Prox) typically require $\mathcal{O}(nm\log(nm))$ operations for
matrices in $\mathbb{R}^{n\times m}$, limiting their computational and
energy efficiency.

In this paper, we introduce Two-Stage Clipping (TSC), an efficient
structured sparsification operator inspired by bilevel optimization.
Rather than solving a bilevel optimization problem in the strict sense,
TSC decomposes the computation into two successive stages. First,
group-level quantities are projected to allocate the global sparsity
budget and identify active groups. Second, the resulting values are used
as group-wise clipping thresholds for the matrix coefficients.
For square attention matrices ($n \times n$),
the complexity of our TSC method is $O(n^2)$, which is linear in the
number of matrix coefficients.

We establish a theoretical bound on the discrepancy between TSC and the
exact proximal projection and show on synthetic data that the additional
Euclidean approximation cost remains small in practice. Extensive
benchmarks demonstrate substantial runtime improvements over proximal
baselines and validate the theoretical norm identities underlying the
proposed construction. 

We derive exact mixed-norm decomposition identities for both the
proximal projection and TSC operators. 
These identities establish an exact norm conservation
and decomposition property shared by the classical
Euclidean projection and the proposed two-stage
clipping operator.

We further extend the two-stage framework to the
$\ell_{1,2}$ structured sparsity constraint. Remarkably, in this
case, we prove that the two-stage construction coincides exactly
with the Euclidean projection onto the $\ell_{1,2}$ ball.
Consequently, the proposed framework provides two complementary
situations: for $\ell_{1,\infty}$, TSC is a computationally efficient
approximation of the exact proximal projection with a quantitative
error bound, whereas for $\ell_{1,2}$, TSC recovers the exact
Euclidean projection.

Experiments on biomedical datasets show
that TSC achieves classification accuracy and sparsity levels comparable
to those of the exact proximal approach when applied to Transformer
attention matrices.
Experiments on Image dataset show
that TSC achieves better PSNR than classical Pruning methods when applied to Transformer
attention matrices.

Overall, TSC provides a computationally efficient alternative to exact
proximal projection for structured sparsification, combining linear
complexity, theoretical guarantees, and competitive empirical
performance.

\end{abstract}
\section{Introduction}

Modern deep neural networks achieve state-of-the-art performance across vision, language, and biomedical applications, but this success comes at the cost of massive parameterization and significant computational overhead. As models continue to scale, enforcing structured sparsity has become a central challenge for improving efficiency, interpretability, and deployability \cite{Alv2016,Han2015-2,Sanh,double,double2}. These approaches yield empirical sparsity but do not enforce explicit geometric constraints. Unstructured sparsity induced by $\ell_1$ regularization often fails to translate into real computational gains on modern hardware \cite{tRS,sparse},\cite{duchi,condat}. \\

Among structured sparsity-inducing norms, the $\ell_{1,\infty}$ norm plays a central role \cite{bejar2021fastest}. It promotes column-wise sparsity by coupling entries within columns while encouraging entire groups to vanish. However, projection onto the $\ell_{1,\infty}$ ball remains computationally expensive. State-of-the-art projection algorithms rely on proximal formulations and the Moreau identity, leading to a worst-case complexity of $\mathcal{O}(nm \log(nm))$ for matrices in $\mathbb{R}^{n\times m}$ due to sorting operations across matrix entries. This complexity becomes prohibitive for large-scale neural networks and Transformer architectures.\\
The authors in \cite{ICTAI} improved the proximal algorithm for computing the exact Euclidean projection onto the $\ell_{1,\infty}$ ball, achieving near-linear complexity $O(nm+J\log(nm))$, where $J$ decreases as sparsity increases. However,  the main drawback is that its computational complexity remains sparsity-dependent.\\

Recently, bilevel optimization has become popular because it provides a powerful framework for modeling problems involving nested objective functions \cite{overview,Bilevel-pruning,BLO}. Bilevel optimization (BLO) involves two levels of hierarchy (i.e., upper and lower levels), where obtaining the solution to the upper-level problem requires solving the lower-level problem.\\
Inspired by this bilevel framework, we propose a two-stage clipping (TSC) operator that replaces the proximal Euclidean projection.

Transformer architectures \cite{Transformer} were 
proposed for learned image compression Liu et al.~\cite{Liu2023TCM}.
Of particular relevance to the work in \cite{Bigbird}
recently investigated sparse attention. In their work  sparsity is induced
by the attention architecture itself.

These works motivate the evaluation of the proposed TSC operators for
Biomedical Transformer based classification and Transformer-based image compression. In our experiments, attention matrices
are sparsified using TSC-$\ell_{1,\infty}$, TSC-$\ell_{1,2}$, and the
corresponding Euclidean proximal projection. The resulting methods are
evaluated on Classification of biomedical datasets.
We provide attention Sparsity–Reconstruction Trade-off on large image datasets.
In particular, we report Delta-PSNR  as
function of the attention sparsity.

\paragraph{Contributions.}
The main contributions of this work are:
\begin{itemize}
    \item New two-stage clipping (TSC) operators for the $\ell_{1,\infty}$ and $\ell_{1,2}$ structured norms.
    \item A single-level convex characterization of the lower-stage clipping operator for the radius vector generated by the upper-stage $\ell_1$ projection.
    \item New norm identities for the TSC $\ell_{1,\infty}$ and $\ell_{1,2}$ structured operators.
    \item A bound on the difference between the PROX and TSC approaches for $\ell_{1,\infty}$, together with a proof that TSC $\ell_{1,2}$ coincides exactly with the Euclidean projection.
    \item Empirical validation on Transformer architectures demonstrating improved computational efficiency and structured sparsity.
\end{itemize}

\section{Background on the $\ell_{1,\infty}$ projection
via proximal methods}

\paragraph{Notation.}
Throughout this paper, we follow the notation in
\footnote{https://www.deeplearningbook.org/contents/notation.html}.
For vectors, $\|\cdot\|_2$, $\|\cdot\|_1$, and
$\|\cdot\|_\infty$ denote the standard Euclidean,
$\ell_1$, and $\ell_\infty$ norms, respectively,
while $\|\cdot\|_{\mathrm F}$ denotes the
Frobenius norm of a matrix.

Projection onto the $\ell_{1,\infty}$ norm ball
is a standard tool for enforcing structured
sparsity \cite{quattoni2009efficient,bejar2021fastest}.

For a matrix $\mY\in\sR^{n\times m}$,
the mixed $\ell_{1,\infty}$ norm is defined as
\begin{equation}
\|\mY\|_{1,\infty}
=
\sum_{j=1}^{m}
\max_{1\leq i\leq n}|Y_{ij}|.
\end{equation}

For a given radius $\eta>0$, the corresponding
convex ball is
\begin{equation}
\mathcal B_{\eta}^{1,\infty}
=
\left\{
\mX\in\sR^{n\times m}:
\|\mX\|_{1,\infty}\leq\eta
\right\}.
\end{equation}

The Euclidean projection of $\mY$ onto this ball
is defined by
\begin{equation}
\label{Proj}
P_{\eta}^{1,\infty}(\mY)
=
\arg\min_{\mX\in\mathcal B_{\eta}^{1,\infty}}
\frac12\|\mX-\mY\|_{\mathrm F}^{2}.
\end{equation}

The dual norm of $\ell_{1,\infty}$ is
$\ell_{\infty,1}$, defined as
\begin{equation}
\|\mY\|_{\infty,1}
=
\max_{1\leq i\leq n}
\sum_{j=1}^{m}|Y_{ij}|.
\end{equation}

The proximity operator associated with the
dual norm is defined by
\begin{equation}
\label{eq:prox_infty_1}
\operatorname{prox}_{\eta\|\cdot\|_{\infty,1}}(\mY)
=
\arg\min_{\mX\in\sR^{n\times m}}
\left\{
\frac12\|\mX-\mY\|_{\mathrm F}^{2}
+
\eta\|\mX\|_{\infty,1}
\right\}.
\end{equation}

Therefore, Moreau's decomposition
\cite{moreau62,Combettes} yields
\begin{equation}
\label{eq:moreau_l1infty}
\boxed{
P_{\eta}^{1,\infty}(\mY)
=
\mY-
\operatorname{prox}_{\eta\|\cdot\|_{\infty,1}}(\mY).
}
\end{equation}

Thus, the same parameter $\eta$ defines both
the projection radius and the regularization
weight in the dual proximal operator. Efficient algorithms for computing the exact
$\ell_{1,\infty}$ projection are described in
\cite{bejar2021fastest}, with further
computational improvements reported in
\cite{ICTAI}. The proximal projection algorithm is provided in appendix \ref{Bejar}.

\section{A new two-stage $\ell_{1,\infty}$ structured clipping operator}
\label{TSC-1Infty}

\subsection{A single-level convex formulation}
In this paper, we propose a two-stage approach inspired by bilevel optimization methods, rather than a bilevel optimization problem in the strict sense \cite{overview,Bilevel-pruning,BLO}. In the first stage, the group-level quantities are projected to identify the active groups, while in the second stage, the resulting projected values are used to constrain or rescale the coefficients within each group.
We propose these two-stage method for enforcing
$\ell_{1,\infty}$-structured sparsity.
We provide a single-level convex characterization of the proposed two-stage operator.\\
Let \(Y\in\mathbb{R}^{n\times m}\) with columns
\(y_1,\dots,y_m\in\mathbb{R}^n\). Define the vector of column \(\ell_\infty\) norms
\begin{equation}
v \;:=\; (\|y_1\|_\infty,\dots,\|y_m\|_\infty)\in\mathbb{R}^m.
\label{eq:def_v}
\end{equation}
Let define the specific two-stage optimization problem as follows:
\begin{equation}
\hat u \;:=\; P_1^\eta(v)
\;\in\;
\arg\min_{u:\ \|u\|_1\le \eta,\ u\ge 0}
\frac12\|u-v\|_2^2,
\label{eq:def_u_hat}
\end{equation}
and define the two-stage operator \(TSC_{1,\infty}^\eta(Y)\) column-wise as
\begin{equation}
x_j \;:=\; P_\infty^{\hat u_j}(y_j),
\qquad j=1,\dots,m.
\label{eq:def_two-stage_x}
\end{equation}

\begin{proposition}[Exact single-level convex optimality]
Let \(\hat u\) be defined by \eqref{eq:def_u_hat}. Then the matrix
\(X = TSC_{1,\infty}^\eta(Y)\) defined by \eqref{eq:def_two-stage_x} is the unique
minimizer of the following single-level convex optimization problem:
\begin{equation}
\boxed{
\min_{X\in\mathbb{R}^{n\times m}}
\quad
\frac12\|X-Y\|_F^2
\quad
\text{s.t.}
\quad
\|x_j\|_\infty \le \hat u_j,
\ \forall j\in\{1,\dots,m\}.
}
\label{eq:single_level_hard_u}
\end{equation}
\end{proposition}

\paragraph{Proof.}
Since \(\hat u\) is fixed, the constraint set in \eqref{eq:single_level_hard_u}
is a Cartesian product of convex sets:
\(
\mathcal{C}(\hat u)
=
\prod_{j=1}^m \{x\in\mathbb{R}^n:\ \|x\|_\infty\le \hat u_j\}.
\)
Therefore, the minimization problem \eqref{eq:single_level_hard_u} decouples
across columns:
\begin{equation}
\min_{X\in \mathcal{C}(\hat u)} \frac12\|X-Y\|_F^2
\;\equiv\;
\sum_{j=1}^m
\min_{\|x\|_\infty\le \hat u_j}
\frac12\|x-y_j\|_2^2.
\label{eq:decouple}
\end{equation}
For each column \(j\), the problem
\(
\min_{\|x\|_\infty\le \hat u_j}\frac12\|x-y_j\|_2^2
\)
is a Euclidean projection onto the closed convex set
\(
B_\infty^{\hat u_j}
=
\{x:\|x\|_\infty\le \hat u_j\}.
\)
Thus its unique minimizer is
\(x_j=P_\infty^{\hat u_j}(y_j)\), which is exactly \eqref{eq:def_two-stage_x}.
Uniqueness follows from strict convexity of the objective in each column.
\noindent
Proposition 3.1 provides an explicit single-level convex
characterization of the second stage of the proposed TSC operator.
For the radius vector $\widehat{\mathbf u}$ determined by the first
stage, TSC is the exact Euclidean projection onto the Cartesian product
of the corresponding column-wise $\ell_\infty$ balls.
This characterization does not imply that TSC coincides with the
Euclidean projection onto the global $\ell_{1,\infty}$ ball.

\subsection{A clipping operator algorithm}

\begin{remark}
\label{rem:l1_bounds}
As a consequence of the non-expansiveness and positivity of the $\ell_1$ projection: $\hat{\vu}
~\gets~ P^{1}_\eta\!\left(
\|\vy_1\|_\infty,\dots,\|\vy_m\|_\infty
\right)$; the components of $\hat{\vu}$ satisfy
\begin{equation}
0 \;\leq\; \hat{u}_j \;\leq\; \|\vy_j\|_\infty,
\qquad \forall j\in\{1,\dots,m\}.
\end{equation}
These bounds hold regardless of the magnitude of the column norms.
\end{remark}
Each column $\vx_j$ is thus computed as the projection of $\vy_j$ onto the
$\ell_\infty$ ball of radius $\hat{u}_j$:
\begin{equation}
\label{eq:linf_proj}
\vx_j
~\gets~
P^{\infty}_{\hat{u}_j}(\vy_j),
\qquad \forall j\in\{1,\dots,m\}.
\end{equation}

This projection admits the closed-form expression
\begin{equation}
\label{eq:linf_clipping}
X_{i,j}
~=~
\operatorname{sign}(Y_{i,j})\,
\min\bigl(|Y_{i,j}|,\hat{u}_j\bigr),
\qquad \forall i,j.
\end{equation}

\begin{remark}[Clipping interpretation]
The mapping $\mY \mapsto TSC^{1,\infty}_\eta(\mY)$ can be interpreted as a
\emph{column-wise clipping operator}, where $\hat{\vu}$ acts as a vector of
clipping thresholds.
Indeed,
\begin{align}
Y_{i,j} - X_{i,j}
&=
\operatorname{sign}(Y_{i,j})
\bigl(|Y_{i,j}| - \min(|Y_{i,j}|,\hat{u}_j)\bigr), \\
|Y_{i,j} - X_{i,j}|
&=
\bigl||Y_{i,j}| - \min(|Y_{i,j}|,\hat{u}_j)\bigr|.
\end{align}
Using Remark~\ref{rem:l1_bounds}, we obtain
\begin{equation}
\max_i |Y_{i,j} - X_{i,j}|
=
\|\vy_j\|_\infty - \hat{u}_j.
\end{equation}
Since $\hat{u}_j = \|\vx_j\|_\infty$ by construction, this yields
\begin{equation}
\label{eq:linf_difference}
\|\vy_j - \vx_j\|_\infty
=
\|\vy_j\|_\infty - \|\vx_j\|_\infty,
\qquad \forall j.
\end{equation}
\end{remark}

Algorithm~\ref{algo:linfbiproj} summarizes a possible implementation of the
TSC $\ell_{1,\infty}$ clipping.

\begin{algorithm}[H]
\caption{TSC $\ell_{1,\infty}$ clipping ($TSC^{1,\infty}_\eta(\mY)$)}
\label{algo:linfbiproj}
\begin{algorithmic}
\STATE \textbf{Input:} $\mY \in \sR^{n\times m}$, $\eta > 0$
\STATE $\hat{\vu} \gets P^{1}_\eta\!\left(\|\vy_1\|_\infty,\dots,\|\vy_m\|_\infty\right)$
\STATE $x_j \;=\; P_{\infty}^{u_j}(y_j)
\quad\Longleftrightarrow\quad
(x_j)_i \;=\; \mathrm{sign}((y_j)_i)\,\min\!\big(|(y_j)_i|,\,u_j\big),
\qquad \forall i.$
\STATE \textbf{Output:} $\mX = (\vx_1,\dots,\vx_m)$
\end{algorithmic}
\end{algorithm}

\begin{remark}
A vector $\vu$ is a feasible clipping operator if 
a feasible radius vector satisfies
\[
\sum_{j=1}^{m}u_j\leq\eta.
\]
Equality holds when the
constraint is active.
\end{remark}

\subsection{An $\ell_{1,\infty}$ norm Identity}
The purpose of this section is to compare the structural properties of the classical
$\ell_{1,\infty}$ projection and the proposed two-stage clipping.
In particular, we study the relationship between the norm of the clipping solution in the spirit of classical $L$-curve
analysis~\cite{Lcurves}. Recall the triangle inequality for the Frobenius norm,
which follows from the Cauchy--Schwarz inequality:
\begin{equation}
\label{eq:triangle}
\|\mY - TSC^{1,\infty}_\eta(\mY)\|_{\mathrm{F}}
+
\|TSC^{1,\infty}_\eta(\mY)\|_{\mathrm{F}}
\;\geq\;
\|\mY\|_{\mathrm{F}}.
\end{equation}

However, for the $\ell_{1,\infty}$ norm, we obtain a stronger identity for our TSC.

\begin{proposition}[Exact $\ell_{1,\infty}$ norm decomposition]
\label{prop:identity_bilevel}
For the TSC $\ell_{1,\infty}$ clipping, the projected matrix and the residual
satisfy the exact identity
\begin{equation}
\label{eq:identity_bilevel}
\boxed{
\|\mY - TSC^{1,\infty}_\eta(\mY)\|_{1,\infty}
+
\|TSC^{1,\infty}_\eta(\mY)\|_{1,\infty}
~=~
\|\mY\|_{1,\infty}.
}
\end{equation}
\end{proposition}

\begin{proof}
The result follows directly by summing over the columns the identity
\begin{equation}
\label{eq:basicInf_recall}
\|\vy_j - \vx_j\|_\infty
=
\|\vy_j\|_\infty - \|\vx_j\|_\infty,
\qquad \forall j,
\end{equation}
which holds for the TSC clipping .
Summing over $j$ yields
\[
\sum_{j=1}^m \|\vy_j - \vx_j\|_\infty
=
\sum_{j=1}^m \|\vy_j\|_\infty
-
\sum_{j=1}^m \|\vx_j\|_\infty,
\]
which is exactly~\eqref{eq:identity_bilevel}.
\end{proof}

\begin{proposition}[Exact identity for the proximal projection ]
\label{prop:identity_classical}
The standard $\ell_{1,\infty}$ projection  satisfies
\begin{equation}
\label{eq:identity_classical}
\|\mY - P^{1,\infty}_\eta(\mY)\|_{1,\infty}
+
\|P^{1,\infty}_\eta(\mY)\|_{1,\infty}
~=~
\|\mY\|_{1,\infty}.
\end{equation}
\end{proposition}

\begin{proof}
The proof follows exactly the same argument as for
Proposition~\ref{prop:identity_bilevel}, using the clipping structure of
$P^{1,\infty}_\eta$.
\end{proof}

\begin{remark}[Exact mixed-norm decomposition]
Equations~(18) and~(20) reveal a common structural property
of the exact Euclidean projection and the proposed two-stage
clipping operator. Both operators provide an exact additive
decomposition of the original mixed norm into a retained
component and a residual component.

In particular, when the constraint is active, both methods
preserve the same mixed-norm budget and yield identical
residual norms:
\begin{equation}
\boxed{
\|Y-TSC(Y)\|_{1,\infty}
=
\|Y-P(Y)\|_{1,\infty}
=
\|Y\|_{1,\infty}-\eta.
}
\end{equation}

Thus, although TSC and the exact projection generally produce
different matrices, they achieve exactly the same global
mixed-norm reduction. Their difference lies in how this
reduction is distributed among the matrix coefficients,
rather than in the total mixed-norm budget.

\end{remark}

\subsection{Relationship between $\ell_{1,\infty}$  Euclidean Projection and Two-Stage Clipping}
\label{sec:prox_tsc}

Let
\[
P=P_{\eta}^{1,\infty}(Y)
\qquad\text{and}\qquad
TSC=TSC_{\eta}^{1,\infty}(Y)
\]
denote, respectively, the exact Euclidean projection and the proposed
two-stage clipping operator.

Since
\[
P=
\arg\min_{\|X\|_{1,\infty}\leq\eta}
\frac{1}{2}\|X-Y\|_F^2
\]
and $TSC$ is feasible, i.e.,
\[
\|TSC\|_{1,\infty}\leq\eta,
\]
the optimality of the Euclidean projection immediately gives
\begin{equation}
\label{eq:optimal_projection}
\|Y-P\|_F
\leq
\|Y-TSC\|_F.
\end{equation}

The following result provides a quantitative bound on the discrepancy
between the two solutions.

\begin{proposition}[Distance between exact projection and two-stage clipping]
\label{prop:distance_TSC_P}
The exact Euclidean projection $P$ and the two-stage clipping solution
$TSC$ satisfy
\begin{equation}
\label{eq:distance_TSC_P}
\boxed{
\|TSC-P\|_F^2
\leq
\|Y-TSC\|_F^2-\|Y-P\|_F^2.
}
\end{equation}
\end{proposition}

Proposition~\ref{prop:distance_TSC_P} shows that the excess squared
Euclidean approximation error of two-stage clipping directly controls
its distance to the exact projection. Hence, when the two approximation
errors are close, the corresponding solutions are also close.

Moreover, both operators admit a column-wise clipping representation.
Let $u_j^P$ and $u_j^{TSC}$ denote their respective clipping thresholds.
Since scalar clipping is $1$-Lipschitz with respect to its threshold,
\[
\left|
\operatorname{clip}(Y_{ij},u_j^P)
-
\operatorname{clip}(Y_{ij},u_j^{TSC})
\right|
\leq
|u_j^P-u_j^{TSC}|.
\]
Summing over all entries gives the deterministic bound
\begin{equation}
\label{deterministic_TSC_bound}
\boxed{
\|P-TSC\|_F^2
\leq
n\sum_{j=1}^{m}
|u_j^P-u_j^{TSC}|^2
=
n\|u^P-u^{TSC}\|_2^2.
}
\end{equation}
Equivalently,
\[
\boxed{
\|P-TSC\|_F
\leq
\sqrt{n}\,
\|u^P-u^{TSC}\|_2.
}
\]
This bound is distribution-free and therefore holds for arbitrary
input data.

\begin{proposition}[Refined deterministic bound]
Let $P=P_\eta^{1,\infty}(Y)$ and
$TSC=TSC_\eta^{1,\infty}(Y)$ denote the exact Euclidean
projection and the two-stage clipping operator,
respectively.

For each column $j$, define
\[
a_j=\min\{u_j^P,u_j^{TSC}\},
\qquad
d_j=|u_j^P-u_j^{TSC}|,
\]
and
\[
k_j(a_j)=\#\{i:|Y_{ij}|>a_j\}.
\]

Since both clipping operators coincide whenever
$|Y_{ij}|\leq a_j$, their difference satisfies

\begin{equation}
\label{difference}
|P_{ij}-TSC_{ij}|
=
\min\left\{
d_j,(|Y_{ij}|-a_j)_+
\right\}.
\end{equation}

For each coefficient, the difference between the
two clipping operators is zero when
$|Y_{ij}|\leq a_j$ and is bounded by $d_j$
otherwise. Therefore,
\begin{equation}
|P_{ij}-TSC_{ij}|^2
\leq
d_j^2\,\mathbf{1}_{\{|Y_{ij}|>a_j\}}.
\end{equation}
Summing over all coefficients gives

\begin{equation}
\label{eq:refined_bound}
\boxed{
\|P-TSC\|_F^2
\leq
\sum_{j=1}^{m}k_j(a_j)d_j^2
\leq
n\|u^P-u^{TSC}\|_2^2.
}
\end{equation}
\end{proposition}

Finally, $k_j(a_j)\leq n$ yields the
previous deterministic bound.
\paragraph{Refined deterministic bound using actual threshold exceedances.}

The deterministic bound can be further refined by exploiting
the fact that the exact Euclidean projection and TSC coincide
for coefficients whose absolute values do not exceed the
smaller clipping threshold.

Summing equation \ref{difference} over all coefficients gives

\begin{equation}
\|P-TSC\|_F^2
=
\sum_{j=1}^{m}\sum_{i=1}^{n}
\min\left\{
d_j^2,(|Y_{ij}|-a_j)_+^2
\right\}.
\end{equation}

Using the elementary inequality
$\sum_i\min\{a_i,b_i\}\leq
\min\{\sum_i a_i,\sum_i b_i\}$,
we obtain the refined deterministic bound
\begin{equation}
\label{eq:refined_deterministic_bound}
\boxed{
\|P-TSC\|_F^2
\leq
\sum_{j=1}^{m}
\min\left\{
k_jd_j^2,
\sum_{i=1}^{n}(|Y_{ij}|-a_j)_+^2
\right\}.
}
\end{equation}

This yields the hierarchy
\begin{equation}
\label{Bound_R}
\begin{aligned}
\|P-TSC\|_F^2
&\leq
\sum_{j=1}^{m}
\min\left\{
k_jd_j^2,
\sum_{i=1}^{n}(|Y_{ij}|-a_j)_+^2
\right\}\\
&\leq
\sum_{j=1}^{m}k_jd_j^2\\
&\leq
n\|u^P-u^{TSC}\|_2^2.
\end{aligned}
\end{equation}

Unlike the previous worst-case estimate, this refined
bound accounts for both the number of coefficients
exceeding the clipping thresholds and the magnitude
of their exceedances. It is deterministic and
distribution-free.

\subsection{Computational complexity}

The proposed two-stage $\ell_{1,\infty}$ clipping
operator is particularly well suited to GPU
architectures. Its computation consists of three
steps (see Algorithm~1):
\begin{enumerate}
    \item Column-wise $\ell_\infty$ norm computation
    with complexity $O(nm)$.
    \item Projection of the resulting vector onto
    the $\ell_1$ ball with complexity $O(m\log m)$.
    \item Element-wise clipping of the matrix
    columns with complexity $O(nm)$.
\end{enumerate}

The computational complexity of the proposed TSC operator is
\begin{equation}
T_{\mathrm{TSC}}(n,m)
= O(nm+m\log m).
\end{equation}
For sufficiently large matrices satisfying
$n=\Omega(\log m)$, the dominant term is $O(nm)$.
Consequently,
\begin{equation}
\boxed{
T_{\mathrm{TSC}}(n,m)=O(nm).
}
\end{equation}
In particular, for square attention matrices ($n=m$),
the complexity is $O(n^2)$, which is linear in the
number of matrix coefficients.
These operations involve parallel reductions,
a single vector projection, and element-wise
transformations, making TSC particularly efficient
on GPUs.

In contrast, classical exact $\ell_{1,\infty}$
projection algorithms require more elaborate
optimization procedures, involving sorting,
threshold computations, and active-set updates
\cite{quattoni2009efficient,bejar2021fastest}.

Although Algorithm~2 of Perez et al.~\cite{ICTAI}
provides an efficient CPU implementation of the exact
$\ell_{1,\infty}$ projection, its heap-based structure
involves sequential dependencies and irregular memory
access patterns that are not naturally suited to
massively parallel GPU execution.
This limitation is related to the algorithmic structure
rather than to the C++ programming language itself,
which can also support efficient GPU implementations
through CUDA.
An important distinction between the proposed TSC
operator and Algorithm~2 of Perez et al.
is their dependence on the induced sparsity level where $J$ denotes the number of heap-update
operations. This quantity is data-dependent and
may vary with the projection radius and the
resulting sparsity level.

\begin{table}[ht]
\centering
\caption{Computational complexity of different
$\ell_{1,\infty}$ projection methods.}
\label{tab:projection_complexity}
\begin{tabular}{lll}
\hline
Method & Complexity & Main operations \\
\hline
Original Bejar et al 
& $O(mn\log(mn))$
& Sorting and threshold updates \\

Perez et al. (Algorithm 2)
& $O(mn+J\log(mn))$
& Inverse ordering and heaps \\

Proposed TSC
& $O(mn)$
& $\ell_1$ projection and clipping \\
\hline
\end{tabular}
\end{table}

\section{A two-stage $\ell_{1,2}$ clipping operator}
\label{subsec:l12_bilevel}
In this section, we extend the proposed TSC framework to the
$\ell_{1,2}$ norms, yielding additional forms of structured sparsity.\\
Let $\mY \in \sR^{n \times m}$ be a matrix with column vectors
$\vy_1,\dots,\vy_m \in \sR^n$.
Define the aggregation vector
\[
\vv_2 \;=\; \bigl(\|\vy_1\|_2,\dots,\|\vy_m\|_2\bigr) \in \sR^m,
\]
which collects the $\ell_2$ norms of the columns of $\mY$.
We define the TSC $\ell_{1,2}$ clipping as the solution of the following
optimization problem:
\begin{equation}
\hat u \;:=\; P_1^\eta(v)
\;\in\;
\arg\min_{u:\ \|u\|_1\le \eta,\ u\ge 0}
\frac12\|u-v\|_2^2,
\label{eq:def_u_hat_2}
\end{equation}
and define the two-stage operator \(TSC_{1,2}^\eta(Y)\) column-wise as
\begin{equation}
x_j \;:=\; P_2^{\hat u_j}(y_j),
\qquad j=1,\dots,m.
\end{equation}

Obviously, propositions~1 holds and provides also an explicit single-level convex optimality principle for the proposed $TSC_{1,2}^\eta(Y)$
two-stage operator.

\begin{algorithm}[ht]
\caption{Two-stage $\ell_{1,2}$ clipping ($TSC^{1,2}_\eta(\mY)$)}
\label{algo:li12proj}
\begin{algorithmic}
\STATE \textbf{Input:} $\mY \in \sR^{n\times m}$, $\eta > 0$
\STATE $\hat{\vu} \gets P^{1}_\eta\bigl(\|\vy_1\|_2,\dots,\|\vy_m\|_2\bigr)$
\STATE $ x_j \;=\; P_{2}^{u_j}(y_j)
\quad\Longleftrightarrow\quad
x_j \;=\;
\begin{cases}
y_j, & \text{if } \|y_j\|_2 \le u_j,\\[4pt]
\dfrac{u_j}{\|y_j\|_2}\,y_j, & \text{if } \|y_j\|_2 > u_j.
\end{cases}  $
\STATE \textbf{Output:} $\mX = [\vx_1,\dots,\vx_m]$
\end{algorithmic}
\end{algorithm}

\subsection{The $\ell_{1,2}$ TSC is an exact Euclidean projection}

\begin{proposition}
For any $\mY\in\mathbb{R}^{n\times m}$ and $\eta>0$,
the proposed two-stage clipping operator coincides
with the exact Euclidean projection onto the
$\ell_{1,2}$ norm ball:
\begin{equation}
\boxed{
TSC_{\eta}^{1,2}(\mY)=P_{\eta}^{1,2}(\mY).
}
\end{equation}
\end{proposition}

\begin{proof}
Let $\mY=[\mathbf y_1,\ldots,\mathbf y_m]$ and
$v_j=\|\mathbf y_j\|_2$.

By construction, TSC first computes
\[
\hat{\mathbf u}=P_\eta^1(\mathbf v),
\]
and then performs column-wise Euclidean projections
onto the $\ell_2$ balls of radii $\hat u_j$:
\[
\mathbf x_j=
\begin{cases}
\dfrac{\hat u_j}{v_j}\mathbf y_j,
& v_j>0,\\[2mm]
0,&v_j=0.
\end{cases}
\]

For any matrix $\mZ$ satisfying
$\|\mZ\|_{1,2}\leq\eta$, let
$r_j=\|\mathbf z_j\|_2$.
By the reverse triangle inequality,
\[
\|\mY-\mZ\|_{\mathrm F}^2
=
\sum_{j=1}^m\|\mathbf y_j-\mathbf z_j\|_2^2
\geq
\sum_{j=1}^m(v_j-r_j)^2
\geq
\|\mathbf v-\hat{\mathbf u}\|_2^2.
\]

Since $\hat u_j\leq v_j$, the TSC construction
attains this lower bound:
\[
\|\mY-TSC_\eta^{1,2}(\mY)\|_{\mathrm F}^2
=
\|\mathbf v-\hat{\mathbf u}\|_2^2.
\]

Thus, TSC minimizes the Euclidean projection
objective over the $\ell_{1,2}$ ball.
By uniqueness of the Euclidean projection,
\[
TSC_\eta^{1,2}(\mY)=P_\eta^{1,2}(\mY).
\]
\end{proof}

Unlike the $\ell_{1,\infty}$ case, the proposed
two-stage construction introduces no approximation
error for the $\ell_{1,2}$ norm.

The TSC-$\ell_{1,2}$ and Prox-$\ell_{1,2}$
algorithms are mathematically identical and
share exactly the same implementation.
Consequently, they produce identical solutions
with identical computational complexity
and execution time.

\subsection{A unified two-stage framework for mixed-norm projections}

The proposed Two-Stage Clipping (TSC) framework provides
a unified computational approach to enforcing mixed-norm
constraints. Its main advantage is that it separates
the projection into two stages: first, a low-dimensional
projection of column-wise norms onto an $\ell_1$ ball,
followed by a column-wise reconstruction satisfying
the prescribed constraint.

Interestingly, the resulting operator exhibits different
optimality properties depending on the underlying
mixed norm.

\paragraph{The $\ell_{1,\infty}$ case: a fast feasible approximation.}

For a matrix $Y\in\mathbb{R}^{n\times m}$, let
\[
a_j=\|Y_j\|_\infty,
\qquad
r=P_\eta^1(a),
\]
where $P_\eta^1$ denotes the Euclidean projection onto
the $\ell_1$ ball of radius $\eta$.

The TSC operator is defined by
\[
\mathrm{TSC}_\eta^{1,\infty}(Y)_{ij}
=
\operatorname{sign}(Y_{ij})
\min\{|Y_{ij}|,r_j\}.
\]

By construction,
\[
\|\mathrm{TSC}_\eta^{1,\infty}(Y)\|_{1,\infty}
\leq \eta.
\]

However, this operator does not generally minimize
the Frobenius distance to $Y$ over the feasible set.
It therefore provides a computationally efficient
feasible approximation to the exact Euclidean projection.

\paragraph{The $\ell_{1,2}$ case: recovery of the exact projection.}

Consider now the mixed norm
\[
\|Y\|_{1,2}=\sum_{j=1}^{m}\|Y_j\|_2.
\]

Let
\[
a_j=\|Y_j\|_2,
\qquad
r=P_\eta^1(a).
\]

The second stage consists of the radial reconstruction
\[
\mathrm{TSC}_\eta^{1,2}(Y)_j
=
\begin{cases}
\dfrac{r_j}{a_j}Y_j, & a_j>0,\\[6pt]
0, & a_j=0.
\end{cases}
\]

Since the projection of $a$ onto the $\ell_1$ ball
satisfies
\[
r_j=(a_j-\lambda)_+,
\]
for some $\lambda\geq0$, we obtain
\[
\mathrm{TSC}_\eta^{1,2}(Y)_j
=
\left(1-\frac{\lambda}{\|Y_j\|_2}\right)_+Y_j,
\]
with the convention that zero columns remain zero.

This expression coincides with the classical
group soft-thresholding formula for the Euclidean
projection onto the $\ell_{1,2}$ ball. Consequently,
\begin{equation}
\boxed{
\mathrm{TSC}_\eta^{1,2}(Y)
=
P_\eta^{1,2}(Y).
}
\end{equation}

\paragraph{Interpretation and contribution.}

These results highlight an important structural property
of the proposed framework. While TSC provides a fast
feasible approximation for the $\ell_{1,\infty}$ constraint,
it recovers the exact Euclidean projection for the
$\ell_{1,2}$ constraint.

The exact $\ell_{1,2}$ projection is a well-established
result in proximal optimization and structured sparsity.
Therefore, the contribution is not the introduction
of a new $\ell_{1,2}$ projection algorithm, but rather
the identification of a common two-stage formulation
that encompasses both approximate and exact projections.

This unified perspective clarifies how the geometry
of the mixed-norm constraint determines whether
the two-stage construction is exact or approximate,
while retaining a computational structure suitable
for vectorized implementations.

\begin{remark}[Exact $\ell_{1,2}$ norm identity]
\label{prop:l12_identity}
The two-stage $\ell_{1,2}$ clipping  satisfies the identity 
\begin{equation}
\label{Identity-L12}
\boxed{
\|\mY - TSC^{1,2}_\eta(\mY)\|_{1,2}
+
\|TSC^{1,2}_\eta(\mY)\|_{1,2}
=
\|\mY\|_{1,2}.
}
\end{equation}
\end{remark}

Consequently, the exact mixed-norm identity also holds
for the Euclidean projection:
\begin{equation}
\boxed{
\|Y-P_{\eta}^{1,2}(Y)\|_{1,2}
+
\|P_{\eta}^{1,2}(Y)\|_{1,2}
=
\|Y\|_{1,2}.
}
\end{equation}

This exact additive decomposition is specific to the mixed
$\ell_{1,2}$ norm and does not, in general, extend to the
Frobenius norm.

\section{Experimental results} 
\label{Exp}

Both TSC and the exact projection are evaluated on the
same synthetic matrices, using identical projection radii,
hardware, and timing protocols.
This comparison therefore evaluates the computational
advantage of TSC over a Prox exact projection on a GPU
implementation.
While the Prox exact projection minimizes the Euclidean distance
to the input matrix, TSC provides a computationally
efficient feasible approximation.
We consequently report both execution times and
approximation errors to characterize the trade-off
between computational efficiency and projection accuracy.

\subsection{Comparison of classical Euclidean projection vs proposed two-stage clipping operator}
We compared our two-stage method against the proximal projection algorithm on synthetic Gaussian data.
\begin{figure}[ht]
    \centering
    \includegraphics[width=0.6\textwidth]{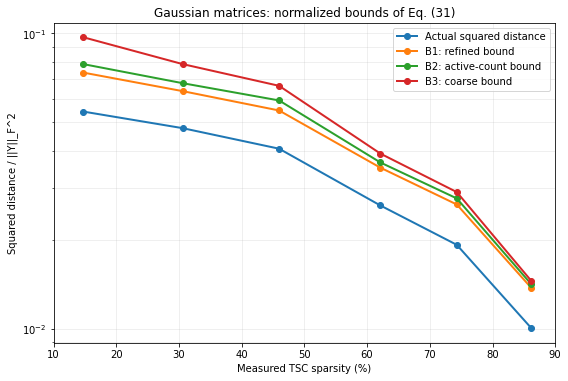}
    \includegraphics[width=0.6\textwidth]{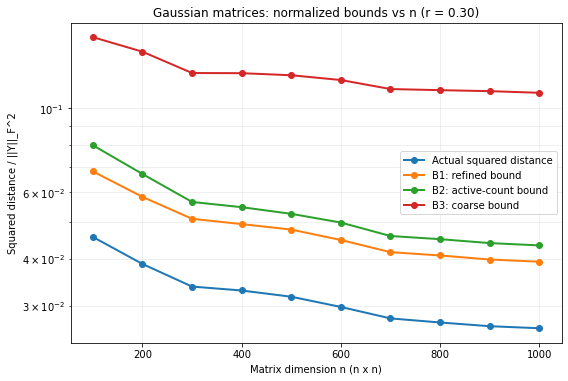}
    \caption{
Relative difference between the classical Euclidean projection
$P_{\eta}^{1,\infty}$ and the proposed two-stage clipping operator
$TSC_{\eta}^{1,\infty}$ for i.i.d.\ Gaussian matrices, as a function of
the size and the sparsity.
}
\label{fig:bounds}
\end{figure}
Figure~\ref{fig:bounds} reports the normalized additional
Euclidean cost of the two-stage clipping operator relative to the exact
projection,
as a function of the size n (Bottom) (for r= 0.3) and the sparsity (Top) for $n=1000$ for Gaussian input matrices.

Empirical results on Gaussian data in Figure~\ref{fig:bounds} show that:
\begin{itemize}
    \item 
The normalized distance between the exact proximal
projection (Prox) and the proposed two-stage clipping
operator (TSC) remains very small;
\item  This distance decreases as the matrix dimension $n$
and the sparsity level increase;
\item The theoretical upper bounds exhibit a similar
behavior with respect to both matrix dimension and sparsity.

\item Furthermore, the refined bounds in Eq. \ref{Bound_R} provide
significantly tighter estimates of the actual distance
than the coarse upper bound.
\end{itemize}
These experimental results empirically support the
similarity between TSC and Prox in terms of
classification accuracy.

\subsection{Time comparison between the proximal projection and our two-stage clipping operator methods on GPU device }
Both TSC and the classical proximal projection (Prox)
were implemented in PyTorch and evaluated on an NVIDIA
A100 GPU. The proximal algorithm is detailed in
Appendix~\ref{Bejar}. Both methods were tested on the same
square matrices using identical projection radii.
Execution times are reported in milliseconds as a
function of the matrix dimension, and the speedup
is defined as the ratio between Prox and TSC
execution times.

\begin{figure}[ht]
    \centering
    \includegraphics[width=0.7\textwidth]{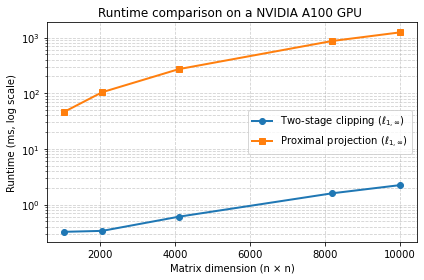}
    \includegraphics[width=0.7\textwidth]{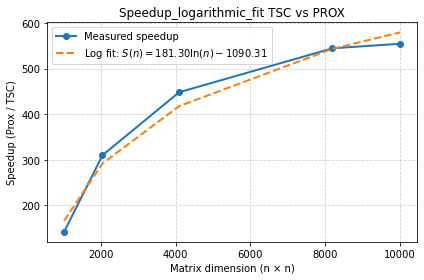}
    \caption{Comparison of two-stage Clipping versus proximal. U: Processing time in PyTorch using an NVIDIA A100 GPU as a function of the size of the square matrix $n \times n $. Bottom: Logarithm fitting of the speedup of our two-stage clipping operator versus the proximal algorithm.}
    \label{A100}    
\end{figure}
Figure~\ref{A100} shows the runtime comparison between the proposed two-stage clipping $\ell_{1,\infty}$  operator 
and the classical proximal $\ell_{1,\infty}$ projection on GPU.
Note that runtime of two-stage clipping operator $\ell_{1,2}$ is identical to the runtime of the two-stage clipping operator $\ell_{1,\infty}$ .
and the classical proximal $\ell_{1,\infty}$
The horizontal axis represents the matrix dimension $n$ for square matrices 
$n \times n$, while the vertical axis (log scale) shows the runtime in milliseconds.
The two-stage method relies only on reductions and element-wise clipping operations,
which are highly parallelizable on GPU architectures.
In contrast, the proximal method requires repeated simplex projections
involving sorting operations, leading to significantly higher runtime.
The results shown in Figure~\ref{A100} (left) indicate that the proposed two-stage clipping algorithm
outperforms the proximal method by a factor ranging up to $540 $ when the size of the attention
matrix is $n=10,000$.
Figure\ref{A100} (Bottom) show that the measured speedup increases approximately
logarithmically over the considered matrix-size range, consistently with
the different computational structures of the two algorithms.

\subsection{Classification on biomedical dataset with sparse Transformer Architectures}

\paragraph{Criterion.}
We define the optimisation problem on transformer architectures.
Let $\mW \in \mathbb{R}^{n\times n}$ denote the attention matrix, where $n$ is the number of tokens. Let $\vz \in \mathbb{R}^{n\times 1}$ represent the true labels, and $\vz^*$ the estimated labels obtained from a softmax classifier. 
To sparsify the weights $\mW$ of the attention matrix, we compare the classical projection and  the two-stage clipping method as a constraint to enforce structured sparsity in the model. The global optimization criterion is defined as:
\begin{equation}
 \label{crit}
\boxed{
\begin{aligned}
\min_{W}\quad & \phi\bigl(z,z^*(W)\bigr) \qquad
\text{subject to}\quad & \|W\|_{1,\infty}\leq \eta.
\end{aligned}
}
\end{equation}
 
where $\phi(\vz,\vz^*(W))$ is the cross-entropy loss.

The optimization problem is solved using a masked double-descent gradient algorithm \cite{double}.
We implemented the proposed two-stage clipping method in PyTorch, including the model architecture, optimizer, learning-rate schedulers, and loss functions.
We use the Adam optimizer \cite{Adam} and the smooth SiLU activation function.\\
We compare our two-stage clipping operator with the classical proximal $\ell_{1,\infty}$ projection for matrix sparsification on  transformer architectures \cite{Transformer}.
We use a lightweight Transformer architecture with two layers and two heads.

\paragraph{Classification on ECG  Dataset }
Accurate classification and interpretability are critical requirements in biomedical applications such as ECG-based heart failure diagnosis.
We evaluate our method on the PhysioNet PTB Diagnostic ECG Database \cite{ECG-data}.
The classification task is formulated as a binary problem with 10,505 abnormal and 4,045 normal ECG recordings.
The signals are preprocessed and segmented, with each segment corresponding to a heartbeat represented by 187 tokens.
 Note that we use a 45 $^{\circ}$  reshape for inducing  a diagonal-wise structured sparsity.

\begin{figure}[ht]
\centering
\includegraphics[width=0.9\textwidth]{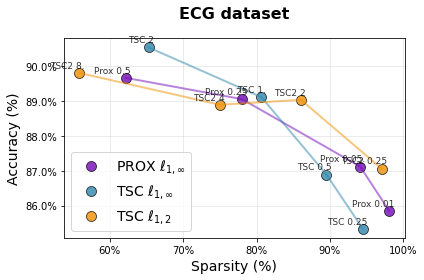}
\caption{ECG dataset. Proximal projection versus two-stage $\ell_{1,\infty}$ clipping operator — sparsity–accuracy trade-off.}
\label{ECG}
\end{figure}

Fig \ref{ECG} shows that, on this biomedical ECG dataset, the proposed
two-stage $\ell_{1,\infty}$ clipping operator and the proximal method
achieve very similar classification accuracy, with a difference of
less than $1\%$.

\paragraph{The sepsis dataset}
The Sepsis dataset \footnote{https://www.kaggle.com/datasets/olagokeblissman/sepsis-dataset} focuses on early warning detection for sepsis, a critical and potentially fatal medical condition. It includes anonymized vital signs, lab results, and clinical indicators of patients admitted to the hospital, structured for real-time monitoring and predictive modeling.  The total number of patients were 40,336 including 2 932 sepsis
 patients and 37,404 patients non-sepsis. We performed balanced sampling at the patient level. All SEPSIS patients were retained and $ 7.27 \%$ of the non-sepsis patients were randomly selected, using a fixed random seed to ensure reproducibility. For the considered dataset, this resulted in approximately 2,932 SEPSIS patients and 2,719 non-sepsis patients, yielding a nearly balanced dataset with a 48-hour temporal sequence length $Tokens = 48)$ and 
40 clinical variables per hour.

\begin{figure}[ht]
\centering
\includegraphics[width=0.9\textwidth]{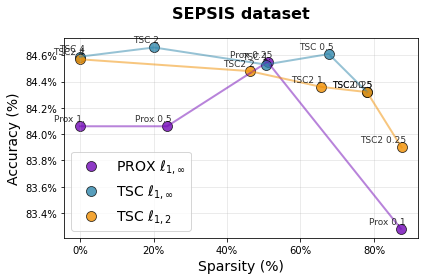}
\caption{SEPSIS dataset. Proximal projection versus two-stage $\ell_{1,\infty}$ clipping operator — sparsity–accuracy trade-off.}
\label{SEPSIS}
\end{figure}

Fig \ref{SEPSIS} shows that, on this biomedical SEPSIS dataset, the proposed
two-stage $\ell_{1,\infty}$ clipping operator and the proximal method
achieve very similar classification accuracy, with a difference of
less than $1\%$.

\subsection{Attention Sparsity–Reconstruction Trade-off (ASRT)}

Let define the optimisation problem on transformer architectures.
Let $\mX$ be the original image, and $\mX^* (w)$  the reconstructed image,
\begin{equation}
 \label{crit-image}
\boxed{
\begin{aligned}
\min_{W}\quad & \phi\bigl(\mX,\mX^*(W)\bigr) \qquad
\text{subject to}\quad & \|W\|_{1,\infty}\leq \eta.
\end{aligned}
}
\end{equation}
 
where $\phi(\mX,\mX^*(W))$ is the quadratic loss. 
The mixed-norm constraint is imposed on the sparsified
core of each attention matrix, excluding the first row
and column, with an adaptive radius
$\eta = r\|A_c\|_{1,\infty}$, where $A_c$ denotes the
original attention submatrix before sparsification
and $0 < r \leq 1$.

In our implementation, the sparsification operator is
applied independently to the post-softmax attention
matrix of each Transformer layer and attention head.

The TSC projection is applied to each mini-batch during every forward pass, and therefore multiple times per epoch.

The first row and column are preserved to maintain
global attention connections, while TSC is applied
to the remaining attention sub-matrix using an
adaptive radius proportional to its mixed norm.

The training loss is the mean squared
reconstruction error.\\
The objective of this experiment is to demonstrate that
TSC-based attention sparsification preserves image reconstruction
quality while increasing attention sparsity, without significantly
affecting the compression bit-rate.
A comprehensive investigation of image compression using
sparse Transformers is left for future work.

We compare the proposed TSC method with the widely used
structured pruning approach \cite{structured}.
We do not include BigBird in this comparison, as it was
primarily designed for large attention matrices.
Furthermore, the proximal projection method is excluded
because of its prohibitive computational cost on
large-scale datasets.

Reconstruction quality is evaluated using the peak
signal-to-noise ratio (PSNR). 
Sparsity levels are matched based on the nominal sparsity parameters while effective attention sparsity is
measured as the proportion of zero-valued attention
coefficients.

Both experiments employ the same Transformer-based
image reconstruction architecture. 
The architectural parameters are summarized in
Table~\ref{archi} in Appendix~\ref{compression}.

\subsubsection{Evaluation of ASRT on the
Kodak image dataset }

We evaluate the proposed structured sparsification operators on the
Kodak image dataset using a lightweight Transformer-based image
autoencoder. The dataset is randomly partitioned into disjoint training
and test subsets, containing $N_{\mathrm{train}}$ and $N_{\mathrm{test}}$
images, respectively, using a fixed random seed.

The resulting sparse model is subsequently evaluated on previously
unseen images using the same operator and constraint.\\

\begin{figure}[ht]
\centering
\includegraphics[width=0.7\textwidth]{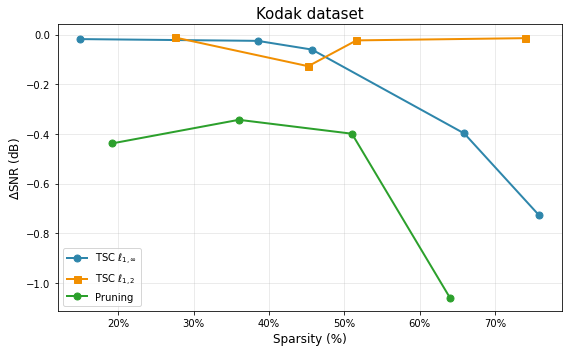}
\caption{Kodak dataset. Two-stage $\ell_{1,\infty}$ clipping operator versus Two-stage $\ell_{1,2}$ clipping operator and structured pruning: sparsity–Delta-PSNR trade-off.}
\label{Kodak}
\end{figure}
The figure shows the PSNR loss as a function of attention sparsity.
The figure \ref{Kodak} shows that TSC-based methods preserve
PSNR better than structured pruning at comparable
sparsity levels. In particular, TSC-$\ell_{1,2}$
exhibits the best reconstruction performance.

\subsubsection{Evaluation of ASRT using Flickr for training and kodak for testing }
We trained the transformer on 473 $2048 \times
2048 $images obtained from Flickr, divided into $256 \times 256$
patches. The dataset contains 473 images, with 30,272 available patches. At each epoch, 3,784 patches are sampled, corresponding to 8 patches per image. Over 30 epochs, this amounts to 113,520 parameter updates, assuming that each patch constitutes a separate batch. We use the 24-image Kodak dataset for testing. We report Sparsity–Delta-PSNR trade-off. 

\begin{figure}[ht]
\centering
\includegraphics[width=0.7\textwidth]{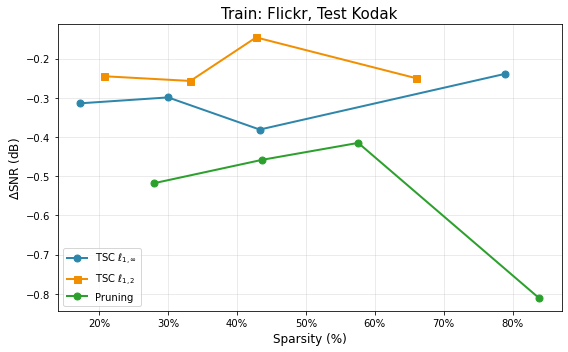}
\caption{Training on Flickr and testing on Kodak dataset. Two-stage $\ell_{1,\infty}$ clipping operator versus Two-stage $\ell_{1,2}$ clipping operator and structured pruning: sparsity–Delta-PSNR trade-off.}
\label{Flickr}
\end{figure}

The figure \ref{Flickr} shows that TSC-based methods preserve
PSNR better than structured pruning at comparable
sparsity levels. In particular, TSC-$\ell_{1,2}$
exhibits the best reconstruction performance.

\section{Discussion}

We could also consider extending TSC to the $\ell_{1,1}$ constraint,
which preserves the same matrix norm identity. However, this extension
presents three limitations:
(i) it induces unstructured sparsity within columns;
(ii) we cannot establish a quantitative error bound relating TSC
to the exact proximal projection;
and (iii) the additional $\ell_1$ projections on individual columns
introduce a logarithmic computational overhead.\\
Future work will focus on extending the proposed TSC framework to image compression and comparing its performance with standard codecs, such as JPEG 2000, as well as recent Transformer-based image compression methods \cite{Ma2026}.

\section{Conclusion}
Classical $\ell_{1,\infty}$ projection methods are typically derived from proximal formulations resulting in algorithms with $\mathcal{O}(nm \log(nm))$ worst-case complexity. \\
In contrast, our approach avoids computing the projection via dual proximal mappings. 
In this paper, we introduce the Two-Stage Clipping (TSC) framework
for enforcing structured sparsity constraints.
The proposed framework provides two complementary approaches:
for $\ell_{1,\infty}$, TSC offers a computationally efficient
approximation of the exact Euclidean projection, supported by
a quantitative error bound, whereas for $\ell_{1,2}$,
TSC coincides exactly with the Euclidean projection.

For square attention matrices ($n \times n$),
the complexity of our TSC method is $O(n^2)$, which is linear in the
number of matrix coefficients.

Experiments show that our two-stage clipping algorithm outperforms the proximal
algorithm in a range of 143 to 554 when the size of the attention
matrix varies within a range of 1000 to 10000.
By exploiting the intrinsic structure of the constraint, the proposed approach connects convex analysis with practical neural network sparsification, enabling scalable structured regularization in modern deep learning models.

\bibliography{references}

\newpage
\appendix
\section{Proximal $(1,\infty)$ algorithm}
\label{Bejar}

Recall the proximal projection algorithm:

\begin{algorithm}[ht]
\caption{Proximal $\ell_{1,\infty}$ Projection }
\label{Prox}
\begin{algorithmic}[1] 
    \STATE $Y \in \sR^{n \times m}_+, \eta > 0$
    \STATE $X = P_{\ell_{1,\infty}}(Y)$
    \STATE $\mathbf{a} \gets \{1,\dots,m\}$
    \STATE $\theta \gets \frac{\sum_j \max y_j-\eta}{m}$
    \WHILE{$\theta$ changed}
        \FOR{$j \in \mathbf{a}$}
            \IF{$\|y_{j}\|_1 < \theta$} \label{alg:ProjC:naive:remove}
                \STATE $\mathbf{a} \gets \mathbf{a}\setminus \{j\}$
                \STATE \textbf{continue}
            \ENDIF
            \STATE $\mathbf{x}_j \gets P_{1}^\theta(y_j)$
            \STATE $S_j \gets \{ i \mid x_{i,j} > 0\}$ \label{alg:ProjC:naive:gather}
        \ENDFOR
        \STATE $\theta \gets \frac{\sum_{j\in\mathbf{a}} \frac{\sum_{i \in S_j} Y_{i,j}}{|S_j|} - \eta} {\sum_{j\in \mathbf{a}}\frac{1}{|S_j|}}$
    \ENDWHILE
    \STATE $\forall j, \mu_j \gets \max(0,\frac{\sum_{i \in S_j} Y_{i,j} - \theta}{|S_j|})$
    \STATE $\forall i,j, X_{i,j} \gets \min(Y_{i,j},\mu_j)$
\end{algorithmic}
\end{algorithm}

\section{Relationship between $\ell_{1,\infty}$  Euclidean Projection and Two-Stage Clipping}
\begin{proof}
Let
\[
\mathcal C_\eta
=
\left\{X:\|X\|_{1,\infty}\leq\eta\right\}.
\]
Since $P$ is the Euclidean projection of $Y$ onto the closed convex set
$\mathcal C_\eta$, its variational characterization gives
\[
\langle Y-P,X-P\rangle_F\leq0,
\qquad
\forall X\in\mathcal C_\eta.
\]
Taking $X=TSC\in\mathcal C_\eta$ yields
\[
\langle Y-P,TSC-P\rangle_F\leq0.
\]
Using
\[
Y-TSC=(Y-P)-(TSC-P),
\]
we obtain
\[
\|Y-TSC\|_F^2
\geq
\|Y-P\|_F^2+\|TSC-P\|_F^2,
\]
which proves~\eqref{eq:distance_TSC_P}.
\end{proof}

\newpage


\section{Attention sparsity-reconstruction trade-off (ASRT)
evaluation framework}
\label{compression}
\begin{table}[ht]
\centering
\caption{Transformer architecture shared by the Kodak-to-Kodak
and Flickr-to-Kodak experiments.}
\label{archi}
\begin{tabular}{lc}
\hline
\textbf{Architecture parameter} & \textbf{Value} \\
\hline
Input image channels             & 3 (RGB) \\
Training image size              & $256 \times 256$ \\
Patch size                       & $16 \times 16$ \\
Number of tokens                 & 256 \\
Embedding dimension              & 256 \\
Number of Transformer layers     & 4 \\
Number of attention heads        & 8 \\
Dimension per attention head     & 32 \\
Attention matrix size (per head) & $256 \times 256$ \\
Feed-forward hidden dimension    & 1024 \\
Feed-forward activation          & GELU \\
Normalization                    & LayerNorm (Pre-LN) \\
Residual connections             & Yes \\
Latent dimension                 & 64 \\
Encoder                          & Patch embedding (Conv2D) \\
Decoder                          & Transposed convolution (ConvTranspose2D) \\
Latent quantization step         & $1/32$ \\
\hline
\end{tabular}
\end{table}

\begin{figure}[t]
\centering
\begin{tikzpicture}[
    >=Latex,
    node distance=0.65cm,
    block/.style={
        draw,
        rounded corners,
        align=center,
        minimum height=0.85cm,
        text width=3.5cm,
        font=\small
    },
    arrow/.style={->, thick}
]

\node[block] (input)
    {Input Image $X$};

\node[block, below=of input] (embed)
    {Patch Embedding};

\node[block, below=of embed] (attention)
    {Multi-Head Self-Attention\\
     TSC / Structured Pruning};

\node[block, below=of attention] (encoder)
    {Transformer Encoder};

\node[block, below=of encoder] (latent)
    {Latent Representation $Z$};

\node[block, below=of latent] (quant)
    {Latent Quantization $Q(Z)$};

\node[block, below=of quant] (decoder)
    {Image Decoder};

\node[block, below=of decoder] (output)
    {Reconstructed Image $\widehat{X}$};

\draw[arrow] (input) -- (embed);
\draw[arrow] (embed) -- (attention);
\draw[arrow] (attention) -- (encoder);
\draw[arrow] (encoder) -- (latent);
\draw[arrow] (latent) -- (quant);
\draw[arrow] (quant) -- (decoder);
\draw[arrow] (decoder) -- (output);

\node[block, right=1.2cm of attention] (sparsity)
    {Attention Sparsity $s$};

\draw[arrow] (attention.east) -- (sparsity.west);

\node[block, right=1.2cm of output] (quality)
    {Reconstruction Quality\\
     $\Delta PSNR(s)$, $\Delta SSIM(s)$};

\draw[arrow] (output.east) -- (quality.west);

\end{tikzpicture}

\caption{
Vertical architecture of the attention sparsity--reconstruction
trade-off (ASRT) evaluation framework.
Attention sparsification is performed using TSC or structured
pruning within the Transformer encoder.
Reconstruction quality is evaluated on the Kodak dataset
using PSNR and SSIM relative to the dense baseline.
No entropy coding is considered.
}
\end{figure}
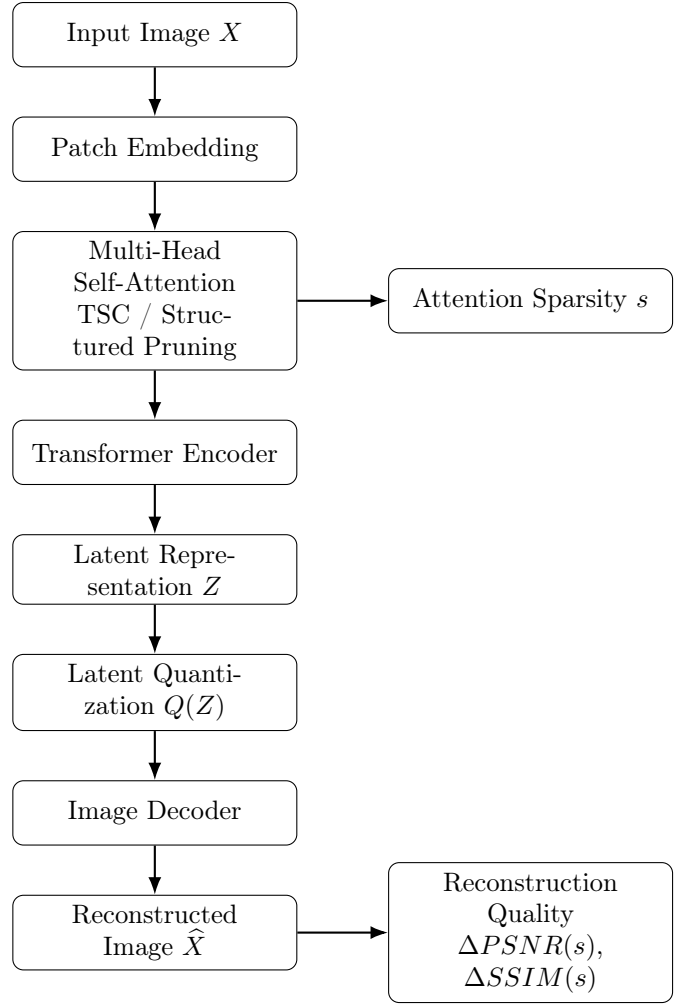

\end{document}